\documentclass[12pt]{article}
\usepackage{amsfonts}
\usepackage{amssymb}
\usepackage{amsmath}
\usepackage{graphicx,psfrag,epsf}
\usepackage{enumerate}
\usepackage{url} 
\usepackage{float}
\usepackage{dsfont,amsfonts,mathrsfs}

\usepackage[]{algorithm2e}
\usepackage{multirow}
\usepackage{lipsum}
\usepackage{amsfonts}
\usepackage{cite}      
                  
\usepackage{graphicx}   

\usepackage{amsmath}
\usepackage{amsthm}
\usepackage{amsfonts}
\usepackage{amssymb}
\usepackage{graphicx}
\usepackage{subfigure}
\usepackage{multirow}
\usepackage{color}
\usepackage{blkarray}
\usepackage{natbib}
\usepackage[dvipsnames]{xcolor}

\newcommand{\bm}[1]{\boldsymbol{#1}}

\renewcommand{\d}{\mathrm{d}}
\renewcommand{\P}{\mathbb{P}}
\newcommand{\1}{\mathbb{I}}

\newcommand{\Y}{\bm{Y}}

\newcommand{\bN}{\mathbb{N}}

\newcommand{\wt}[1]{\widetilde{#1}}

\newcommand{\iid}{\stackrel{iid}{\sim}}

\newcommand{\wh}[1]{\smash{\widehat{#1}}}

\def\C {\,|\:}

\def\C {\,|\:}

\def\a{\bm{b}}
\def\mF{\mathcal{F}}

\def\B{\bm{B}}
\def\b{\bm{\beta}}
\def\Y{\bm{Y}}

\def\x{\bm{x}}

\def\W{\bm{W}}

\def\bg{\bm{\gamma}}

\def\b{\bm{\beta}}

\renewcommand{\d}{\mathrm{d}\,}
\newcommand{\e}{\mathrm{e}}

\newcommand{\R}{\mathbb{R}}
\newcommand{\Ha}{\mathcal{H}^\alpha}

\newcommand{\N}{ \mathbb{N} }

\newcommand{\vnorm}[1]{\left|\left|#1\right|\right|}

\newcommand{\bx}{{\bf x}}

\newtheorem{lemma}{Lemma}[section]

\newtheorem{theorem}{Theorem}[section]

\newtheorem{remark}{Remark}[section]
\newtheorem{corollary}{Corollary}[section]

 \theoremstyle{assumption}

\begin{document}

\def\spacingset#1{\renewcommand{\baselinestretch}%
{#1}\small\normalsize} \spacingset{1}


  \title{\sf  Posterior Concentration for Sparse Deep Learning}

 \author{
       Nicholas Polson\footnote{
    Robert Law, Jr. Professor of Econometrics and Statistics at the {\sl \small Booth School of Business, University of Chicago}} \hspace{0.3cm} and \hspace{0.3cm}
    Veronika Ro\v{c}kov\'{a}\footnote{  Assistant Professor in Econometrics and Statistics and James S. Kemper Faculty Scholar at the  {\sl \small Booth School of Business, University of Chicago}}
    }
 
  \maketitle

\bigskip

\begin{abstract}
\noindent Spike-and-Slab Deep Learning (SS-DL) is a fully Bayesian  alternative to Dropout for improving generalizability of deep ReLU networks. This new type of regularization enables  provable recovery of smooth input-output maps with {\sl unknown} levels of smoothness. 
  Indeed, we  show that  the posterior distribution concentrates at the near minimax rate for $\alpha$-H\"{o}lder smooth maps, performing as well as if we knew the smoothness level $\alpha$ ahead of time.
   Our result sheds light on architecture design for deep neural networks, namely the choice of depth, width and sparsity level.
   These network attributes typically depend on  unknown smoothness  in order to be optimal. We obviate this constraint with the fully Bayes construction.
   As an aside, we show that SS-DL does not overfit in the sense that the posterior concentrates on smaller networks with fewer (up to the  optimal number of) nodes and links.
Our results provide new theoretical justifications for deep ReLU networks from a Bayesian point of view.
\end{abstract}

\noindent%
{\it Keywords:}  Deep Learning,  Non-parametric Bayes, Posterior Concentration, Rectified Linear Units, Sparsity,  Spike-and-Slab
\vfill

\newpage
\spacingset{1.45} 

\section{Introduction}

Deep learning constructs are powerful tools for pattern matching and prediction. Their empirical success has been accompanied by a number of theoretical developments addressing (a) why and when neural networks generalize well, (b) when do deep networks out-perform shallow ones and  (c) which activation functions and  with how many layers.  Despite the flurry of research activity, there are still many theoretical gaps in  understanding why deep neural networks work so well. 
In this paper, we provide several new insights by studying  the speed of posterior concentration around the optimal predictor, and in doing so we make a contribution to the Bayesian literature on deep learning rates. 

Bayesian non-parametric methods are proliferating rapidly in statistics and machine learning, but their theoretical study has not yet kept pace with their application. 
Lee (2000), for example, showed consistency  of posterior distributions over single-layer sigmoidal neural networks. Our contribution builds on this line of research in  three fundamental aspects: (a) we focus on deep rather than single-layer, (b) we focus on rectified linear units (ReLU) rather than sigmoidal squashing functions, (c) we show that the posterior converges at an optimal speed beyond the mere fact that it is consistent.
 To achieve these goals, we adopt a statistical perspective on deep learning through the lens of  non-parametric regression. 

Using deep versus shallow networks can be justified theoretically in a number of ways. First, while both shallow and deep neural networks (NNs) are universal approximators (i.e. can approximate any continuous multivariate function arbitrarily well on a compact domain), Mhaskar et al. (2017) show that deep nets can use  exponentially fewer f parameters to achieve the same level of approximation accuracy for compositional functions. Second, Kolmogorov (1963) showed that superpositions of univariate semi-affine functions provide a universal basis for  multivariate functions.
Telgarsky (2016) provides examples of functions that cannot be represented efficiently with shallow networks and
Kawaguchi et al (2017) explains why deep networks generalize well.  In related  work, Poggio et al. (2017)  show how deep networks can avoid the curse of dimensionality for compositional functions. These theoretical results are growing and our goal is to show how they can be leveraged to show posterior concentration rates for deep learning. In particular, we will build on approximation properties of deep ReLU networks  characterized recently  by Schmidt-Hieber (2017).

Deep ReLU activating functions can  also be justified theoretically. 
Evidence exists  that training deep learning proceeds best when  neurons are either off or operate in a linear way. Glorot et al. (2011) show that
ReLU  functions outperform hyperbolic tangent or sigmoid squashing functions,  both in terms of statistical and computational performance. The success of ReLUs has  been partially attributed in their ability to avoid vanishing gradients and their expressibility properties.   The attractive approximation properties are discussed in Telgarsky (2017) who shows that there exists a ReLU network for 
approximating any rational function whose size is polynomial  in $ \log_2 ( 1/\varepsilon) $ versus polynomial in $1/\varepsilon $, given the approximation error  $\varepsilon$. Vitushkin (1964) showed that one needs non-differentiable functions for the hidden layers to be able to fully approximate any function. 
Montufar et al. (2014) provide a theoretical estimate of the number of linear regions that ReLU networks can synthesize. 
Schmidt-Hieber (2017) points out a curious aspect of ReLU activators that their composition can yield rate-optimal reconstructions of smooth functions of an arbitrary order, not only up to order 2 which would be expected from piecewise linear approximators. 

Dinh  et al. (2017) states that ``explaining why deep learning can generalize well, despite their overwhelming capacity, is an open area of research". 
It is also commonly perceived that generalizability of neural networks can be improved with regularization (Goodfellow et al., 2016). 
Regularization, loosely defined as any modification to a learning algorithm that is intended to reduce its test error but not its training error (Goodfellow et al., 2016) can be achieved in many different ways. The choice of the activation function (ReLU, in particular) is one possible avenue which we will analyze.

Another way to regularize a neural network is by adding noise to the learning process. For example, Dropout regularization (Srivastava et al., 2014) samples from (and averages over) thinned networks obtained by randomly dropping out nodes together with their connections. While motivated as  stochastic regularization,  Dropout can be regarded  as deterministic $\ell_2$ regularization obtained by margining out  Dropout noise (Wager, 2014).
Dropout averaging over sparse architectures pertains, at least conceptually, to Bayesian model averaging under spike-and-slab priors.
 Spike-and-slab regularization assigns a prior distribution over sparsity patterns (models) and performs model averaging with posterior model probabilities as weights (George and McCulloch, 1993). Similar to Dropout, spike-and-slab effectively switches off model coefficients.
However, Dropout averages out patterns using equal weights rather than  posterior model probabilities. 

Our approach embeds   $\ell_0$ penalization within the layers of  deep learning  and capitalizes on its connection to  subset selection.
Our goal is then to exploit spike-and-slab constructions not necessarily as a tool for model selection, but rather as 
a fully Bayesian alternative to dropout in order  to
 (a) inject sparsity in deep learning to build  stable network architectures, (b) achieve
 adaptation to the unknown aspects of the regression function in order to achieve near-minimax performance for estimating smooth regression surfaces. 

The rest of the paper is outlined as follows. Section 2 describes our statistical framework for analyzing deep learning predictors.
Section 3 defines deep ReLU networks. 
Section 4 constructs an appropriate spike-and-slab regularization for deep learning.  Section 5 provide posterior concentration results for
sparse deep ReLU networks and reviews function approximation rates. Finally, Section  6 concludes with a discussion.

\subsection{Notation}
The $\varepsilon$-covering number of a set $\Omega$ for a semimetric $d$, denoted by $\mathcal{E}(\varepsilon; \Omega; d),$ is the minimal number of $d$-balls of radius $\varepsilon$ needed to cover set $\Omega$. The notation $\lesssim$  will be used to denote inequality up to a constant.

\section{Deep Learning: A Statistical Framework}
Deep Learning, in its simplest form, reconstructs high-dimensional input-output mappings. To fix notation, 
let $Y\in\R$ denote a (low dimensional) output and $ \x = (x_1 , \ldots , x_p )' \in [0,1]^p $  a (high dimensional) set of inputs.  

From a machine learning viewpoint,  predicting an outcome from a set of features is typically framed as  noise-less non-parametric regression for recovering $f_0:[0,1]^p\rightarrow\R$. Given inputs $ \x_i $ of training data and outputs $Y_i=f_0(\x_i)$ for  $1\leq i\leq n$, 
the goal is to learn a deep learning architecture $\wh f^{DL}_{\B}$ such that $\wh f^{DL}_{\B} ( \bx ) \approx f_0 ( \bx ) $ for $ \bx \notin \{ \bx_i \}_{i=1}^n $.
Training neural networks is then positioned as an optimization problem for finding values $\wh\B\in\R^T$ that minimize empirical risk ($L^2$-recovery error on training data)
together with a regularization term, i.e. 
\begin{equation}\label{train}
\wh\B=\arg\min\limits_{\B} \; \sum_{i=1}^n [f_0( \bx_i ) - f^{\text{DL}}_{\B}  ( \bx_i ) ]^2 + \phi (\B ) 
\end{equation}
where $ \phi(\B) $ is a  penalty over the weights and offset parameters $ \B$. In practice, this is most often carried out with 
some form of stochastic gradient descent (SGD) (see e.g. Polson and Sokolov (2017) for an overview).

From a statistical viewpoint,  deep learning is often embedded within non-parametric regression  where responses are linked to fixed predictors in a stochastic fashion through 
\begin{equation}\label{lik}
Y_i=f_0(\x_i)+\varepsilon_i,\quad \varepsilon_i\iid\mathcal{N}(0,1),  \quad 1 \leq i \leq n .
\end{equation}
We define by $\Ha_p=\{f:[0,1]^p\rightarrow\R;\|f\|_{\Ha}<\infty\}$  the class of $\alpha$-H\"{o}lder smooth functions on a unit cube $[0,1]^p$ for some $\alpha>0$, 
where $\|f\|_{\Ha}$ is the H\"{o}lder norm.  The true generative model,  giving rise to \eqref{lik}, will be denoted with $\P_{f_0}^{(n)}$.  Assuming $f_0\in\Ha_p$, we want to reconstruct $f_0$
 with $\wh f^{DL}_{\B}$ so that the empirical $L^2$ distance
$$
\|\wh f^{DL}_{\B}-f_0\|_n^2=\frac{1}{n}\sum_{i=1}^n[\wh f^{DL}_{\B}(\x_i)-f_0(\x_i)]^2
$$
is at most a constant multiple away from the minimax rate $\varepsilon_n=n^{-\alpha/(2\alpha+p)}$ (up to a log factor).
Unlike  related statistical developments (Schmidt-Hieber (2017), Bauer and Kohler (2017)),  we approach the reconstruction problem from a purely Bayesian point of view.
While the optimization problem \eqref{train} has a Bayesian interpretation as MAP estimation under regularization priors, here we study the behavior of the {\sl entire posterior}, not just its mode.

Our approach rests on careful constructions of prior distributions $\pi(f^{DL}_{\B})$ over deep learning architectures. In Bayesian non-parametrics, the  quality of priors can be often quantified with the speed  at which the posterior distribution shrinks around the true regression function as $n\rightarrow\infty$. These statements are ultimately framed in a frequentist way, describing the typical behavior of the posterior under the true generative model $\P_{f_0}^{(n)}$. Posterior concentration rate results are now 
entering the machine learning community as a tool for (a) obtaining more insights into Bayesian methods (van der Pas and Rockova (2017), Rockova and van der Pas (2017)) 
and (b)  prior calibrations. These results quantify not only the typical distance between a point estimator (posterior mean/median) and the truth, but also the typical spread of the posterior around the truth. Ideally, most of the posterior mass should be concentrated in a ball centered around the true value $f_0$ with a radius proportional to the minimax rate $\varepsilon_n$. 
Adopting this perspective, we study posterior concentration  for  deep learning, providing new theoretical justifications for neural networks from a Bayesian point of view.

In the  construction of deep learning priors, a few key questions emerge.
How does one choose the architecture $f^{DL}_{\B}$: how deep and what activation functions? The choice typically depends on how quickly one can reconstruct $f_0$. 
We focus on deep ReLU networks, motivated by the following example.

\subsection{Motivating Example} Mhaskar et al. (2017, remark 8) shows that the bivariate function $  f_{10} (x_1,x_2) = (x^2_1 x^2_2 - x_1^2 x_2 + 1)^{ 2^{10}} $ can be approximated more
efficiently by a deep ReLU  neural net than a shallow combination of ridge functions. To verify this observation, we simulate data from the following polynomial
$$
f_1 (x_1,x_2) =(x^2_1 x^2_2 - x^2_1 x_2 + 1)^2
$$
where $(x_1,x_2)$ take values in $[-1, 1]^2$. We discretize the grid for a total training data of $ 201 \times 201 = 40401 $ observations.

There exists an exact Kolmogorov representation for this function as a superposition of semi-affine functions if we use the identities
for the inner polynomial functions
\begin{align}
x_1^2 x_2 & = \frac{1}{2} ( x_1^2 + x_2 )^2 - \frac{1}{2} ( x_1^2 - x_2 )^2\\
( x_1 x_2 )^2 & = \frac{1}{4} ( x_1 + x_2 )^4 + \frac{7}{4 \cdot 3^3} ( x_1 - x_2)^4
- \frac{1}{2 \cdot 3^3} ( x_1 + 2 x_2 )^4 - \frac{2^3}{3^3} ( x_1 + 2 x_2 )^4. 
\end{align}
Following the theoretical results of Mhaskar et al. (2017), we build an $11$-layer deep ReLU network is used to approximate this polynomial. There are 9 units in the first hidden layer and 3 units in the further layers. All activation functions are ReLU. For comparison, we also build a shallow network with only 1 hidden layer but 2048 units.

The MSE for the models, both trained with SGD in {\tt TensorFlow} and {\tt Keras}, are: 
$11$ layers, $39$ units with  
$MSE(train) = 0.0229, MSE(validation) = 0.0112 $ and 
$1$ layer, $2048$ units with  
$ MSE(train) = 0.0441, MSE(validation) = 0.09 $. 
Both models outperform random forests.

\section{Deep ReLU Networks}

We now formally describe the generative model  that  gives rise to deep rectified linear unit networks.
To fix notation, we write a deep neural network $ f_{\B}^{DL} ( \bx ) $ as an iterative mapping specified by hierarchical layers of abstraction.
With $L\in\bN$ we denote the number of hidden layers and  with  $p_l\in \bN$ the number of neurons
at the $l^{th}$ layer. Setting $p_0=p$ and $p_{L+1}=1$, we denote   with $\bm p=(p_0,\dots, p_{L+1})'\in\bN^{L+2}$ the vector of neuron counts for the entire network.
The deep network is then characterized by a set of model parameters
\begin{equation}\label{parmb}
\B=\{(\W_1,\a_1),(\W_2,\a_2),\dots, (\W_L,\a_L)\},
\end{equation}
  where $\a_l\in\R^{p_{l}}$ are  shift vectors and $\W_L$ are $p_l\times p_{l-1}$ weight matrixes 
  that link neurons between the $(l-1)^{th}$ and $l^{th}$ layers. Nodes in the ReLU network are connected through the following activation function 
 $\sigma_{\a}:\R^r\rightarrow\R^r$  
  $$
  \sigma_{\a}\left(
  \begin{matrix}
  y_1\\
  y_2\\
  \vdots\\
  y_r
  \end{matrix}\right)=
  \left(\begin{matrix}
  \sigma(y_1-b_1)\\
  \sigma(y_2-b_2)\\
  \vdots\\
  \sigma(y_r-b_r)
  \end{matrix}\right),
  $$
  where  $\sigma(x)=ReLU(x) = \max(x,0) $ denotes the rectified linear unit activation function.
  
Deep ReLU neural networks with $L$ layers and a vector of $\bm p$ hidden nodes define an input-output map
 $f_{\B}^{DL}(\x):\R^p\rightarrow\R$  of the form
   \begin{equation}\label{NN}
   f_{\B}^{DL}(\x)=\W_{L+1}\sigma_{\a_L}\left(\W_{L}\sigma_{\a_{L-1}}\hdots \sigma_{\a_1}(\W_1\x)\right).
  \end{equation}

The  representation \eqref{NN} casts neural networks as nested embeddings that allow to express the data flow through a network using variable-size data structures.
Varying the number of active neurons allows a model to control the effective dimensionality for a given input and achieve desired approximation accuracy. 
Following Schmidt-Hieber (2017), we focus on a specific type of networks with an equal number of hidden neurons, i.e. $p_l=12p N$  for each $1\leq l\leq L$  for some $N\in\bN$. 
 We will see  later in Section \ref{sec:post_conc}, that the optimal network width multiplier $N$ should relate to the dimensionality $p$ and smoothness $\alpha$.



\section{Spike-and-Slab Regularization}

 We focus on uniformly bounded $s$-{\sl sparse} deep nets with bounded parameters
  $$
 \mF(L,\bm p,s)=\left\{f_{\B}^{DL}(\x)\,\,\text{as in \eqref{NN}}:  \|f_{\B}^{DL}\|_\infty<F\,\,\text{and}\,\, \|\B\|_\infty\leq 1\,\,\text{and}\,\,\|\B\|_0\leq s \right\},
 $$
where $s\in\bN$ is the sparsity level, i.e. an upper bound on the number of edges in the network, and where $F>0$.

The amount of regularization needed to achieve optimal performance typically depends on unknown properties of functions one wishes to approximate such as their smoothness, compositional pattern and/or the number of variables they depend on.   Hierarchical Bayes procedures have the potential to become fully adaptive and achieve (nearly) minimax performance, as if one knew  these properties ahead of time.  We will leverage the fully Bayes framework and devise a hierarchical  procedure which can learn the optimal level of sparsity needed to achieve near-minimax rates of posterior convergence of neural networks. The cornerstone of this development will be the spike-and-slab framework.

Denote with 
 \begin{equation}\label{eq:numpar}
T=\sum_{l=0}^{L}p_{l+1}(p_l+1)-p_{L+1}<(12\,p\, N)^{L+1}
\end{equation}
the number of parameters in a fully connected network  with $L$ layers and a vector of $\bm p$ neurons.
We treat  the  stacked vector of model coefficients $\B=(\beta_1,\dots, \beta_T)'$ in \eqref{parmb} as a random vector arising from  the  {\sl spike-and-slab} prior defined hierarchically through
\begin{equation}\label{ssdl}
\pi(\beta_j\C\gamma_j)=\gamma_j\wt\pi(\beta_j)+(1-\gamma_j)\delta_0(\beta_j),\quad\text{where}\quad \wt\pi(\beta)=\frac{1}{2}\1_{[-1,1]}(\beta)
\end{equation}
is a uniform prior on an interval $[-1,1]$. Here, $\delta_0(\beta)$ is a Dirac spike at zero and  $\gamma_j\in\{0,1\}$ indicting whether or not $\beta_j$ is nonzero. Now  collate the binary indicators into a vector $\bg=(\gamma_1,\dots,\gamma_T)'\in\{0,1\}^T$ that encodes the connectivity pattern.  We assume that, given the sparsity level $s=|\bg|$, all architectures are equally likely a-priori, i.e.
\begin{equation}\label{prior:gamma}
\pi(\bg\C s)=1/{{T\choose s}}.
\end{equation}
The sparsity level $s$ will be first treated as fixed and later assigned a prior with exponential decay. The spike-and-slab construction, defined by \eqref{ssdl} and \eqref{prior:gamma}, has been  studied in linear models  by Castillo and van der Vaart (2012) and in trees/forests by Rockova and van der Pas (2017), who showed that with a suitable prior on $s$, the posterior can adapt to the unknown level of sparsity. We conclude a very similar property for our proposed {\sl spike-and-slab deep learning} (SS-DL) procedure.

It is worthwhile to point out that the prior in \eqref{ssdl} effectively zeroes out individual links rather than entire groups of links attached to one node. The second approach was explored by Ghosh and Doshi-Velez (2017), who suggested assigning a Horseshoe prior on the node preactivators, diminishing influence of individual neurons. The Dropout procedure is also motivated as erasing nodes rather than links. 

\section{Posterior Concentration for Deep Learning}\label{sec:post_conc}

Reconstruction of a function $f_0$ from the training data $ (Y_i, \bx_i)_{i=1}^n $ can be achieved using a Bayesian posterior.
This requires placing a prior measure $\Pi(\cdot)$ on   $\mF(L,\bm p, s)$,  the set of qualitative guesses of  $f_0$. Given observed data $\Y^{(n)}=(Y_1,\dots, Y_n)'$,  inference about $f_0$ is then carried out via the posterior distribution
$$\Pi(A\C\Y^{(n)})=\frac{\int_A \prod_{i=1}^n \Pi_f(Y_i\C\x_i)\d\Pi(f)}{\int \prod_{i=1}^n \Pi_f(Y_i\C\x_i)\d\Pi(f)}\quad\forall A\in\mathcal{B}$$
where $\mathcal{B}$ is a $\sigma$-field on $\mF(L,\bm p, s)$ and 
where $\Pi_f(Y_i\C\x_i)$ is the likelihood function for the output $Y_i$ under $f$.

Our goal is to determine
\emph{how fast the posterior probability measure concentrates around $f_0$} as $n\rightarrow\infty$?
This speed can be assessed by inspecting the size of the smallest  $\|\cdot\|_n$-neighborhoods around $f_0$ that contain most of the posterior probability (Ghosal and van der Vaart, 2007). For a diameter $\varepsilon>0$ and some $M>0$, we denote with 
$$
A_{\varepsilon, M}=\{f^{DL}_{\B}\in\mF(L,\bm p, s):\|f^{DL}_{\B}-f_0\|_n\leq M\,\varepsilon \}
$$ 
the $M\varepsilon$-neighborhood centered around $f_0$. Our goal is  to show that
\begin{equation}\label{concentration}
\Pi(A_{\varepsilon_n, M_n}^c \C \Y^{(n)} ) \rightarrow 0\quad\text{in $\P_{f_0}^{(n)}$-probability as $n\rightarrow\infty$} 
\end{equation}
 for any $M_n\rightarrow\infty$ and for $\varepsilon_n\rightarrow 0$  such that $n\,\varepsilon_n^2\rightarrow\infty$. We will position our results using  $\varepsilon_n=n^{-\alpha/(2\alpha+p)}\log^\delta(n)$ for some $\delta>0$, the near-minimax rate for a $p$-dimensional $\alpha$-smooth function. Proving techniques for statements of type \eqref{concentration} 
 were established in several pioneering works including Ghosal, Ghosh and van der Vaart (2000), Ghosal and van der Vaart (2007), Shen and Wassermann (2001), Wong and Shen (1995), Walker et al. (2007).


The statement \eqref{concentration} can be proved by verifying the following three conditions (suitably adapted from Theorem 4 of Ghosal and van der Vaart (2007)):

\begin{equation}\label{eq:entropy1}
\displaystyle \sup_{\varepsilon > \varepsilon_n} \log \mathcal{E}\left(\tfrac{\varepsilon}{36}; A_{{\varepsilon},1}\cap\mF_n; \|.\|_n\right) \leq n\,\varepsilon_n^2
\end{equation}
\begin{equation}\label{eq:prior1}
\displaystyle {\Pi(A_{\varepsilon_n,1})}\geq \e^{-d\,n\,\varepsilon_n^2 }
\end{equation} 
\begin{equation}\label{eq:remain1}
\displaystyle \Pi(\mathcal{F} \backslash \mathcal{F}_n) = o(\e^{-(d+2)\,n\,\varepsilon_n^2})
\end{equation}
 for some $d>2$.
Above,  $\mathcal{F}_n\subseteq\mF(L,\bm p, s)$ is an approximating space (sieve) that captures the essence of the parameter space. Condition \eqref{eq:entropy1} restricts the size of the model as measured by the Le Cam dimension (or local entropy). The Le Cam dimension, defined here in terms of the log-covering number of $A_{\varepsilon,1}\cap \mF_n$, gives rise to the minimax rate of convergence under certain conditions (Le Cam, 1973). The sieve should not be too large (Condition \eqref{eq:entropy1}), it should be rich enough to approximate  $f_0$ well and it should receive most of the prior mass  (Condition \eqref{eq:remain1}).



The prior concentration (Condition \eqref{eq:prior1}) is needed to ensure that the prior rewards shrinking neighborhoods of $f_0$. This requirement balances with Condition \eqref{eq:entropy1}. The richer the model class (i.e. the more layers/neurons), the better the approximation to $f_0$. It is essential that the prior is supported on models that are good approximators, but that do not overfit.
It is commonly agreed that the approximation gap should be no larger than a constant multiple of $\varepsilon_n$. Below, we review some known results about expressibility of neural networks to get insights into how many layers/neurons are needed to achieve the desired level of  approximation accuracy.

\subsection{Function Class Approximation Rates}

There is an extensive literature on the approximation properties of neural nets. Many tight approximation results are available for simple functions such as indicators
$ f(\x) = \1_B (\bx ) $ where $ B$ is a unit ball (Cheang and Barron, 2000) or a half-space (Cheang (2010), Kainen et al. (2003, 2007) and K\r{r}kova et al. (1997)).
Recent results on the efficiency of ridge NNs (which arise as shallow learners of the form $f=\sum_{j=1}^n a_j \sigma(w_j^Tx-b_j)$ for sigmoidal $\sigma(\cdot)$) are available in 
Ismailov(2017), Klusowki and Barron (2016, 2017). Pinkus (1999) and Petrushev (1999) provide some of the early bounds.

In general, one tries to characterize the asymptotic behavior of the approximation error as follows:
\begin{equation}\label{approx:size}
\|f_0 -  \wh f  \| = \mathcal{O} ( N^{- \frac{\alpha}{p}} ) \; \iff \; \|{ f_0-   \wh f  }\| \leq \varepsilon  \; {\rm where} \;  N =\mathcal{O} ( \varepsilon^{- \frac{p}{\alpha}} ), 
\end{equation}
where $f_0$ is a real-valued $\alpha$-smooth function, $\wh f$ is the neural-network reconstruction and where $N$ is the ``size" of the network (typically the number of hidden nodes).
Different bounds can be obtained for different classes of $f_0$ and different norms $\|\cdot\|$.
The goal is to assess how complex  the network ought to be for it to approximate $f_0$ well (up to a constant multiple of $\varepsilon_n$).

 For deep networks, one also wants to find the asymptotic behavior of the approximation error as a function of depth, not only its size.
The following Lemma will be an essential building block in the proof of our main theorem. It summarizes the expressibility of deep ReLU networks by linking their approximation error (when estimating H\"{o}lder smooth functions)  to the network depth, width and sparsity.

\begin{lemma}(Schmidt-Hieber, 2017)\label{lemma_approx}
Assume that $f_0\in\Ha_p$ for some $\alpha>0$. Then for any $N\geq (\alpha+1)^p\vee(\|f_0\|_\Ha+1)$ there exists a neural network $\wh f\in  \mF(L^\star,\bm p_N^{L^\star}=
(p,12pN,\dots, 12pN,1),s^\star)'$ with 
\begin{equation}\label{lstar}
L^\star=8+(\lfloor\log_2 (n)\rfloor+5)(1+\lceil \log_2 p\rceil)
\end{equation}
layers and sparsity level $s^\star$ satisfying
\begin{equation}\label{sstar}
s^\star\leq 94\, p^2 (\alpha+1)^{2p} N\,(L^\star+\lceil \log_2 p\rceil )
\end{equation}
such that 
$$
\|\wh f-f_0\|_\infty\leq (2\|f_0\|_{\Ha}+1)3^{p+1}\frac{N}{n}+\|f_0\|_{\Ha}2^\alpha N^{-\alpha/p}.
$$
\end{lemma}
\proof
Apply Theorem 3 of Schmidt-Hieber (2017) with $m=\lfloor\log_2(n)\rfloor$. 

\begin{remark}
In a related result, Yarotsky (2017) shows that there exists a ReLU network that satisfies $\| f - \wh f^{DL} \|_\infty \leq \varepsilon  $ with sparsity $ s=c \cdot \varepsilon^{- \frac{p}{\alpha}} / \log_2 ( 1 / \varepsilon ) + 1 $ 
and depth $  L=c \cdot ( \log_2 ( 1/\varepsilon ) + 1 )  $ where $ c = c ( p , \alpha ) $. 
Petersen and Voigtlaender (2017) extend this result to $L^2$-smooth functions. 
\end{remark}

We assume that $p=O(1)$ as $n\rightarrow\infty$. Lemma \eqref{lemma_approx} essentially states that in order to approximate an $\alpha$-H\"{o}lder smooth function with an error  that is at most a constant multiple of $\varepsilon_n$, we have to choose $L\propto\log(n)$ layers with sparsity $s\leq C_S\lfloor n^{p/(2\alpha+p)}\rfloor$. This follows by setting $N=C_N\lfloor n^{p/(2\alpha+p)}/\log(n)\rfloor$.

\subsection{Posterior Concentration for Sparse ReLU Networks}

We now formalize large sample statistical properties of posterior distributions over ReLU networks.  First, we consider a hierarchical prior distribution on $\mF(L,\bm p,s)$, keeping $L$, $\bm p$ and $s$ fixed as if they were known. The prior distribution now only consists of the prior on the connectivity pattern \eqref{prior:gamma} and the spike-and-slab prior on the weights/offsets \eqref{ssdl}.

Our first result provides guidance for calibrating Bayesian deep sparse  ReLU networks (choosing the sparsity level and the number of neurons) {\sl when the level of smoothness $\alpha$ is known}. 
The result can be regarded as a Bayesian analogue of Theorem 1 of Schmidt-Hieber (2017), who showed near-minimax rate-optimality of a sparse multilayer ReLU network estimator that minimizes empirical least-squares. This was the first result on rate-optimality of deep ReLU networks in non-parametric regression, obtained assuming that the sparsity $s$ is known and that the function $f_0$ is a composition of H\"{o}lder functions. 
We build on this result and show that the {\sl entire posterior distribution} for deep sparse ReLu neural networks is concentrating at the near-minimax rate, when $\alpha$ is known  and when $f_0$ is a H\"{o}lder smooth function.  In the next section, we provide an adaptive result which no longer requires the knowledge of $\alpha$.

\begin{theorem}\label{thm1}
Assume $f_0\in\Ha_p$, where $p=O(1)$ as $n\rightarrow\infty$, $\alpha<p$ and $\|f_0\|_\infty\leq F$. 
Let  $L^\star$ be as in \eqref{lstar}, $s^\star$ as in \eqref{sstar} and  $\bm p^\star=(p,12pN^\star,\dots,12pN^\star,1)'\in\N^{L^\star+2}$, where  
$N^\star=C_N\,\lfloor n^{p/(2\alpha+p)}/\log(n)\rfloor$.
Then the posterior probability concentrates at the rate $\varepsilon_n=n^{-\alpha/(2\alpha+p)}\log^\delta(n)$  for $\delta>1$ in the sense that
\begin{equation}\label{conc}
\Pi(f_{\B}^{DL}\in \mF(L^\star,\bm p^\star,s^\star): \|f-f_0\|_n>M_n\,\varepsilon_n\C\Y^{(n)})\rightarrow 0
\end{equation}
in $\P_0^n$ probability as $n\rightarrow\infty$ for any  $M_n\rightarrow\infty$.

\end{theorem}

\begin{proof}
Section \ref{sec:proof:thm1}
\end{proof}
\begin{remark}
Theorem \ref{thm1}  continues the line of theoretical investigation of Bayesian machine learning procedures. 
Lee (2000) obtained posterior consistency for single-layer sigmoidal networks. van der Pas and Rockova (2017) and Rockova and van der Pas (2017) obtained concentration results for Bayesian regression trees and forests.
\end{remark}

\subsection{Adaptation to Smoothness}
Theorem \ref{thm1} was conceived for network architectures that are optimally tuned for $\alpha$ that is fixed as if it were known. However, such oracle information is rarely available, rendering the result less relevant for practical design of networks. In this section, we devise a  hierarchical prior construction (by  endowing  the unknown network parameters with suitable priors), under which
the posterior performs  as well as  if we knew $\alpha$.

From the previous section (and discussion in Schmidt-Hieber (2017)), we know that  the number of layers $L$ can be chosen without the knowledge of smoothness $\alpha$. We will thus continue to assume that the number of layers  is  fixed and equal to $L^\star$ in \eqref{lstar}. 

Both the network width $N$ and sparsity level $s$ were chosen in an $\alpha$-dependent way. To obviate this constraint, we treat them as unknown with the following priors.
For  the network width multiplier $N$, we  deploy
\begin{equation}\label{prior:N}
\pi(N)=\frac{\lambda^N}{(\e^{\lambda}-1)N!}\quad\text{for}\quad N=1,2,\dots\quad\text{for some}\quad \lambda\in \R.
\end{equation}
The prior \eqref{prior:N} is one of the classical complexity priors used frequently in the Bayesian non-parametric literature (Coram and Lalley (2006), Liu et al. (2017), Rockova and van der Pas (2017)).
Similarly, the sparsity level $s$ will be now treated as unknown with the following prior
\begin{equation}\label{prior:s}
\pi(s)\propto  \e^{-\lambda_s s},\quad s=0,1,\dots, T.
\end{equation}

Denote with $\bm p_N^{L^\star}=(p, 12pN,\dots, 12pN,1)'\in\N^{L^\star}$ the now random vector of network widths that depend on $N$ and $L^\star$. Our parameter space now consists of shells of {\sl sparse} deep nets with different widths and sparsity levels, i.e.
 $$
 \mF(L^\star)=\bigcup_{N=1}^\infty\bigcup_{s=0}^{T} \mF(L^\star,\bm p_N^{L^\star},s),
 $$
where $T$ is the number of links in a fully connected network (defined in \eqref{eq:numpar}). We will design an approximating sieve as follows:
\begin{equation}\label{sieve}
\mF_n=\bigcup_{N=1}^{N_n}\bigcup_{s=0}^{s_n} \mF(L^\star,\bm p_N^{L^\star},s)
\end{equation}
for some suitable $N_n\in \N$ and $s_n\leq T$. Following  our discussion earlier in this section, the sieve $\mF_n$ should be rich enough to include networks that approximate well. To this end, we choose $N_n$ and $s_n$ similar to the ``optimal choices" obtained from the fixed $\alpha$ case, i.e.
\begin{equation}\label{Ns}
N_n=\lfloor \wt C_N n^{p/(2\alpha+p)}\log^{2\delta-1}(n)\rfloor\asymp n\varepsilon_n^2/\log n ,\quad \text{and}\quad s_n=\lfloor L^\star N_n\rfloor \asymp n\varepsilon_n^2
\end{equation}
for $\wt C_N>0$. With these choices, we show that the posterior distribution concentrates at the same rate as before, but without assuming $\alpha$.

\begin{theorem}\label{thm2}
Assume $f_0\in\Ha_p$, where $p=O(1)$ as $n\rightarrow\infty$, $\alpha<p$, and $\|f_0\|_\infty\leq F$. 
Let  $L^\star$ be as in \eqref{lstar} and assume priors \eqref{prior:s} and \eqref{prior:N}.
Then the posterior probability concentrates at the rate $\varepsilon_n=n^{-\alpha/(2\alpha+p)}\log^\delta(n)$ in the sense that
\begin{equation}
\Pi(f_{\B}^{DL}\in \mF(L^\star): \|f_{\B}^{DL}-f_0\|_n>M_n\,\varepsilon_n\C\Y^{(n)})\rightarrow 0
\end{equation}
in $\P_0^n$ probability as $n\rightarrow\infty$ for any  $M_n\rightarrow\infty$.
\end{theorem}
\proof Section \ref{sec:proof:thm2}.

We conclude with the following key corollary that shows that deep ReLU networks with adaptive spike-and-slab priors {\sl do not overfit} in the sense that the posterior probability of using  more than the optimal number of nodes and links goes to zero as $n\rightarrow\infty$

\begin{corollary}\label{cor}
Let $N_n$ and $s_n$ be the optimal $\alpha$-dependent choices of $N$ and $s$ defined in \eqref{Ns}.
Under the assumptions in Theorem \ref{thm2} we have
\begin{equation}
\Pi(N>N_n\C\Y^{(n)})\rightarrow 0\quad\text{and}\quad \Pi(s>s_n\C\Y^{(n)})\rightarrow 0
\end{equation}
in $\P_0^n$ probability as $n\rightarrow\infty$.
\end{corollary}
\proof 
This statement follows from Lemma 1 of Ghosal and van der Vaart (2007) and holds upon the satisfaction of the conditions
$$
\Pi(N>N_n)=o(\e^{-(d+2)n\varepsilon_n^2})\quad\text{and}\quad \Pi(s>s_n)=o(\e^{-(d+2)n\varepsilon_n^2})
$$
that are verified in Section \ref{sec:proof:thm2}.

\section{Discussion}
Spike-and-Slab Deep Learning (SS-DL) with ReLU activation has been shown to be a fully Bayes deep learning architecture that can adapt to unknown smoothness.
It gives rise posteriors that concentrate around smooth functions at the near-minimax rate.
The key ingredients for this result are (a) sparsity through spike-and-slab regularization, (b) complexity priors on the network width and sparsity level. 
Spike-and-slab regularization provides a theoretically sound alternative to Dropout regularization.

In sum, there are many non-parametric methods that can achieve near-minimax recovery of H\"{o}lder smooth functions, but the appeal of deep learning is
their compositional structure, making them ideal for regression surfaces that are themselves compositions. Indeed, there is evidence that
deep learning has an exponential advantage over shallow networks for approximating compositions. Schmidt-Hieber (2017), for example, showed that sparsely connected deep ReLU networks achieve a near-minimax rate in learning for compositions of smooth functions. It is possible to adapt our techniques to obtain a Bayesian analogue of his compositional result.


\section{Proofs}
\subsection{Proof of Theorem \ref{thm1}}\label{sec:proof:thm1}
We proof the theorem by verifying Condition \eqref{eq:entropy1} and  \eqref{eq:prior1}, setting $\mF_n= \mF(L^\star,\bm p^\star,s^\star)$.
First, we need to verify the entropy condition and show that
\begin{equation}\label{eq:entropy}
\displaystyle \sup_{\varepsilon > \varepsilon_n} \log \mathcal{E}\left(\tfrac{\varepsilon}{36}, \{f_{\B}^{DL} \in \mF(L^\star,\bm p^\star,s^\star): \|f - f_0\|_n < \varepsilon\}, \|.\|_n\right) \leq n\,\varepsilon_n^2.
\end{equation}

We can upper-bound the local entropy \eqref{eq:entropy} with the global metric entropy.
In addition, 
$$
\{f_{\B}^{DL} \in \mF(L^\star,\bm p^\star,s^\star): \|f\|_\infty\leq \varepsilon\}\subset\{f_{\B}^{DL} \in \mF(L^\star,\bm p^\star,s^\star): \|f\|_n\leq \varepsilon\},
$$
provides an  upper-bound to \eqref{eq:entropy} with
\begin{align*}
&\log N\left(\tfrac{\varepsilon_n}{36}, f_{\B}^{DL} \in \mF(L^\star,\bm p^\star,s^\star), \|.\|_{\infty}\right)\leq 
(s^\star+1)\log\left(\frac{72}{\varepsilon_n} (L^\star+1)(12pN+1)^{2(L^\star+2)}\right)\\
&\quad\lesssim n^{p/(2\alpha+p)}\log(n)\log\left(n/\log^{\delta}(n)\right)\lesssim n^{p/(2\alpha+p)}\log^2(n)\lesssim n\varepsilon_n^2
\end{align*}
for $\delta>1$, where we have used Lemma 10 of Schmidt-Hieber (2017) and the fact that $s^\star \lesssim n^{p/(2\alpha+p)}$  and $N \asymp n^{p/(2\alpha+p)}/\log(n)$.
This verifies the entropy Condition \eqref{eq:entropy1}.

Next, we want to show that the prior concentrates enough mass around the truth in the sense that, for some $d>2$, 
\begin{equation}\label{eq:prior}
\displaystyle {\Pi(f_{\B}^{DL} \in  \mF(L^\star,\bm p^\star,s^\star):  \|f_{\B}^{DL} - f_0\|_n \leq  \varepsilon_n)}\geq \e^{-d\,n\,\varepsilon_n^2 }.
\end{equation} 
Choosing $N^\star=C_N\,\lfloor n^{p/(2\alpha+p)}/\log(n)\rfloor$ in Lemma \ref{lemma_approx},
 there exists a neural network $\wh f_{\wh\B}\in \mF(L^\star,\bm p^\star,s^\star)$ consisting of $\bm p^\star$ nodes aligned in $L^\star\lesssim \log(n)$ layers and
indexed by $\|\wh \B\|_0=s^\star\lesssim n^{p/(2\alpha+p)}\log(n)$
nonzero parameters
such that 
$$
\|\wh f_{\wh\B}-f_0\|_n\leq C_\infty n^{-\alpha/(2\alpha+p)}\log^{\delta\alpha/p}(n)\lesssim \varepsilon_n/2.
$$
The approximation $\wh f_{\wh\B}$ sits on a  network architecture characterized by a specific pattern $\wh\bg$ of nonzero links among $\wh\B$, i.e. $\wh W_l$ and $\wh a_l$ for $1\leq l\leq L+1$.  We denote by $ \mF(\wh\bg,L^\star,\bm p^\star,s^\star)\subset \mF(L^\star,\bm p^\star,s^\star)$ all the functions  supported on this particular architecture. These functions differ only in the size of the $s^\star$ nonzero coefficients among $\B$, denoted by $\b\in\R^{s^\star}$. With $\wh\b$, we denote the $s^\star$-vector  associated with the nonzero elements in $\wh\B$.

Notice that there are ${T\choose s^\star}\leq (12\,p\, N)^{(L^\star+1)\,s^\star}$ combinations to pick $s^\star$ the nonzero coefficients and each one, according to prior \eqref{prior:gamma},  has an equal prior probability of occurence $1/{T\choose s^\star}$.

To continue, we note (from the triangle inequality) that
$$
\{f_{\B}^{DL} \in \mF(L^\star,\bm p^\star,s^\star):\|f_{\B}^{DL}-f_0\|_n\leq \varepsilon_n\}\supset \{f_{\B}^{DL} \in \mF(\wh \bg):\|f_{\B}^{DL}-\wh{f}_{\wh\B}\|_\infty\leq \varepsilon_n/2\}.
$$
Next, we denote with $\{\b\in\R^{s^\star}:\|\b\|_\infty\leq 1\quad\text{and}\quad\|\b-\wh \b\|_\infty\leq \varepsilon_n\}$ the set of coefficients that are at most $\varepsilon$-away from the best approximating coefficients $\wh\b$ of the neural network $\wh f_{\wh\B}\in \mF(\wh\bg,L^\star,\bm p^\star,s^\star)$.
From the proof of Lemma 10 of Schmidt-Hieber (2017), it follows that 
\begin{align*}
&\left\{f_{\B}^{DL} \in \mF(\wh \bg):\|f_{\B}^{DL}-\wh{f}_{\wh\B}\|_\infty\leq \frac{\varepsilon_n}{2}\right\}\supset\\
&\quad\quad\quad\quad\quad\left\{\b\in\R^{s^\star}:\|\b\|_\infty\leq 1\,\,\text{and}\,\,\|\b-\wh \b\|_\infty\leq \frac{\varepsilon_n}{2V(L^\star+1)}\right\},
\end{align*}
where $V=\prod_{l=0}^{L^\star+1}(p_l^\star+1)$.
Now we have all the pieces  needed to  find a lower bound to the probability in \eqref{eq:prior}. We can write, for some suitably large $C>0$,
\begin{align*}
&{\Pi\left(f_{\B}^{DL} \in  \mF(L^\star,\bm p^\star,s^\star):  \|f_{\B}^{DL} - f_0\|_n \leq  \varepsilon_n\right)}>\frac{{\Pi(f_{\B}^{DL} \in  \mF(\wh\bg,L^\star,\bm p^\star,s^\star):  \|f_{\B} - \wh f_{\wh\B}\|_\infty \leq  \varepsilon_n /2)}}{{T\choose s^\star}}\\
 &> \e^{-(L^\star+1) s^\star\, \log (12\,p\,N^\star)} \Pi\left(\b\in\R^{s^\star}:\|\b\|_\infty\leq 1\,\,\text{and}\,\,\|\b-\wh \b\|_\infty\leq \frac{\varepsilon_n}{2V(L^\star+1)}\right).\label{lower}
 \end{align*}
To continue to lower-bound the expression above, we note that 
$$
\e^{-(L^\star+1) s^\star\, \log (12\,p\,N^\star)} >\e^{-C\log^2(n) n^{p/(2\alpha+p)}}
$$
for some $C>0$. Under the uniform prior distribution on a cube $[-1,1]^{s^\star}$ we can write
\begin{align*}
& \Pi\left(\b\in\R^{s^\star}:\|\b\|_\infty\leq 1\,\,\text{and}\,\,\|\b-\wh \b\|_\infty\leq \frac{\varepsilon_n}{2V(L^\star+1)}\right)=\left(\frac{\varepsilon_n}{2V(L^\star+1)}\right)^{s^\star}\\
&\geq \e^{-s^\star(L^\star+2)\log(12\, p\, n/\log^{\delta}(n))}\geq  \e^{-D\,n^{p/(2\alpha+p)}\log^2(n)}
\end{align*}
for some $D>0$. We can now combine this bound with the preceding expressions  to conclude that $e^{-(C+D)\,n^{p/(2\alpha+p)}\log^2(n)}\geq \e^{-d\,n\,\varepsilon_n^2}$  for $\delta>1$ and $d>C+D$. 
 This concludes the proof of \eqref{conc}.

\subsection{Proof of Theorem \ref{thm2}}\label{sec:proof:thm2}

First we show that the sieve  $\mF_n$ defined in \eqref{sieve} is still reasonably small in the sense that the log covering number can be  upper-bounded by a constant multiple of $n^{p/(2\alpha+p)}\log^{2\delta}(n)$.
It follows from the proof of Theorem \ref{thm1} that the global metric entropy satisfies
\begin{align*}
 \mathcal{E}\left(\tfrac{\varepsilon_n}{36}, \mF_n, \|.\|_n\right) &\leq \sum_{N=1}^{N_n}\sum_{s=0}^{s_n}
\e^{(s+1)\log\left(\frac{72}{\varepsilon_n} (L^\star+2)(12pN+1)^{2(L^\star+2)}\right)}\\
&\lesssim N_n\,s_n\,\e^{C\,(L^\star+2)(s_n+1)\log (pN_nL^\star/\varepsilon_n)}
\end{align*}
 for  some $C>0$ and thereby
$$
\log \mathcal{E}\left(\tfrac{\varepsilon_n}{36}, \mF_n, \|.\|_n\right)\lesssim \log N_n+\log s_n+ n\,\varepsilon_n^2\lesssim n\,\varepsilon_n^2.
$$
This verifies Condition \eqref{eq:entropy1}.

Next, we need to show that the prior charges the sieve in the sense that
$
\Pi[\mF_n^c]=o(\e^{(d+2)n\varepsilon_n^2})$ for some $d>2$ (determined below). We have
$$
\Pi[\mF_n^c]<\Pi(N>N_n)+\Pi(s>s_n).
$$
We apply the Chernoff bound to find that
\begin{equation}\label{chernoff}
\Pi(N>N_n)<\e^{-t\,(N_n+1)}\mathbb{E}\, \e^{t\,N}\propto  \e^{-t\,(N_n+1)}\left(\e^{\e^t\lambda}-1\right)
\end{equation}
for any $t>0$. With our choice $N_n=\lfloor \wt C_Nn^{p/(2\alpha+p)}\log^{2\delta-1}n\rfloor$  and  with $t=\log N_n$ we obtain
$$
\Pi(N>N_n)\e^{(d+2)\,n\varepsilon_n^2}\lesssim \e^{-(N_n+1)\log N_n+ \lambda N_n+(d+2)\,n\varepsilon_n^2}\rightarrow 0
$$
for a large enough constant $\wt C_N$.
Next, we find that
$$
\Pi(s>s_n)\e^{(d+2)\,n\varepsilon_n^2}\lesssim\e^{-C_s(\lfloor L^\star N_n\rfloor+1)+(d+2)\,n\varepsilon_n^2}\rightarrow 0
$$
for some suitably large $\wt C_N>0$. This verifies Condition \eqref{eq:remain1}.

Finally, we verify the prior concentration Condition \eqref{eq:prior1}.
For $N^\star< N_n$ and  $s^\star<s_n$ we know from the proof of Theorem \ref{thm1} that 
$$
\displaystyle {\Pi(f_{\B}^{DL} \in  \mF(L^\star,\bm p^\star,s^\star):  \|f_{\B}^{DL} - f_0\|_n \leq  \varepsilon_n)}\geq \e^{-D_1\, n\,\varepsilon_n^2}
$$
for some $D_1>2$.  Our priors put enough mass at the ``right choices" $(N^\star,s^\star)$ in the sense that $\pi(N^\star)\gtrsim \e^{-N_n\log (N_n/\lambda)}\gtrsim \e^{-D\, n\varepsilon_n^2}$ and  $\pi(s^\star)\gtrsim\e^{-D\,n\varepsilon_n^2}$ for some suitable $D>0$. Then we can write 
\begin{align*}
&{\Pi(f_{\B}^{DL} \in  \mF_n:  \|f_{\B}^{DL} - f_0\|_n \leq  \varepsilon_n)}\\
&\quad\quad\geq \pi(N^\star)\pi(s^\star) {\Pi(f_{\B}^{DL} \in  \mF(L^\star,\bm p^\star,s^\star):  \|f_{\B}^{DL} - f_0\|_n \leq  \varepsilon_n)}
\geq \e^{-(2D+D_1)\, n\varepsilon_n^2}.
\end{align*} 
With these considerations, we conclude the proof of Theorem \ref{thm2}.

\section{References}

\medskip




\noindent Bauer, B. and Kohler, M. (2017). On Deep Learning as a remedy for the curse of dimensionality in nonparametric regression. \emph{Technical report}.\medskip

\noindent Castillo, I.  and van der Vaart (2012). Needles and straw in a haystack: Posterior concentration for possibly sparse sequences.  \emph{Annals of Statistics}, 40, 2069-2101.\medskip

\noindent Cheang, G. H. (2010). Approximation with neural networks activated by ramp sigmoids.  \emph{Journal of Approximation Theory}, 162, 1450-1465.\medskip

\noindent Cheang, G. H., and Barron, A. R. (2000). A better approximation for balls.  \emph{Journal of Approximation Theory}, 104, 183-203.\medskip

\noindent Coram, M. and Lalley, S. (2010). Consistency of Bayes estimators of a binary regression function.  \emph{Annals of Statistics}, 34, 1233-1269.\medskip

\noindent Dinh, R., Pascanu, R., Bengio, S. and Bengio, Y. (2017). Sharp Minima Can Generalize For Deep Nets.  \emph{arXiv:1703.04933}.\medskip


\noindent George, E.I. and McCulloch, R. (1993). Variable selection via Gibbs sampling. \emph{Journal of the American Statistical Association}, 88, 881-889.\medskip

\noindent Ghosal, S., Ghosh, J. and van der Vaart, A. (2000). Convergence rates of posterior distributions. \emph{Annals of Statistics}, 28, 500-531.\medskip

\noindent Ghosal, S. and van der Vaart, A. (2007). Convergence rates of posterior distributions for noniid observations. \emph{Annals of Statistics}, 35, 192-223.\medskip

\noindent Ghosh, S. and Doshi-Velez, F.  (2017).  Model selection in Bayesian neural networks via horseshoe priors. \emph{Advances in Neural Information Processing Systems}.\medskip

\noindent Glorot, X., Border, A. and Bengio, Y.  (2011).  Deep sparse rectifier neural networks. \emph{Proceedings of the 14th International Conference on Artificial Intelligence and Statistics}.\medskip

\noindent Goodfellow, I., Bengio, Y. and Courville, A.  (2016).  Deep Learning. \emph{MIT Press}.\medskip


\noindent Ismailov, V. (2017). Approximation by sums of ridge functions with fixed directions.  \emph{St. Petersburg Mathematical Journal}, 28, 741-772.\medskip

\noindent Kainen, P. C., K{\r u}rkov\'{a}, V., and Vogt, A. (2003). Best approximation by linear combinations of characteristic functions of half-spaces.  \emph{Journal of Approximation Theory}, 122, 151-159.\medskip

\noindent Kainen, P. C., K\r{r}kov\'{a}, V., and Vogt, A. (2007). A Sobolev-type upper bound for rates of approximation by linear combinations of Heaviside plane waves.  \emph{Journal of Approximation Theory}, 147, 1-10.\medskip

\noindent Kawaguchi, K.,  Kaelbling, L. P. and  Bengio, Y.  (2017). Generalization in Deep Learning. \emph{arXiv:1710.05468}.\medskip

\noindent Klusowki, J.M. and Barron, A.R. (2016). Risk bounds for high-dimensional ridge function combinations including neural networks. \emph{arXiv:1607.01434}.\medskip

\noindent Klusowki, J.M. and Barron, A.R. (2017). Minimax lower bounds for ridge combinations including neural networks. \emph{arXiv:1702.02828}.\medskip 

\noindent Kolmogorov, A. (1963). On the representation of continuous functions of many variables by superposition of continuous functions of one variable and addition. {\emph American Mathematical Society Translation}, 28, 55-59.\medskip



\noindent K\r{r}kov\'{a}, V., Kainen, P. C., and Kreinovich, V. (1997). Estimates of the number of hidden units and variation with respect to half-spaces.  \emph{Neural Networks}, 10, 1061-1068.\medskip

\noindent Le Cam, L. (1973). Convergence of estimates under dimensionality restrictions. \emph{Annals of Statistics}, 1, 38-53.\medskip

\noindent Lee, H. (2000). Consistency of posterior distributions for neural networks. \emph{Neural Networks}, 13, 629-642.\medskip


\noindent Liu, L., Li, D. and  Wong, W.H. (2017). Convergence rates of a partition based Bayesian multivariate density estimation methods. \emph{Advances in Neural Information Processing Systems}, 30.\medskip

\noindent Mhaskar, H. N. (1996). Neural networks for optimal approximation of smooth and analytic functions.  \emph{Neural Computation}, 8(1), 164-177.\medskip

\noindent Mhaskar, H., Liao, Q., and Poggio, T. A. (2017). When and why are deep networks better than shallow ones? In  \emph{AAAI}, 2343-2349.\medskip


\noindent Montufar, G.F., R. Pascanu, K. Cho and Y. Bengio (2014). On the number of linear regions of deep neural networks. \emph{Advances in Neural Information Processing Systems}, 27, 2924-2932.\medskip

\noindent   van der Pas, S. and Rockova, V. (2017). Bayesian dyadic trees and histograms for regression. \emph{Advances in Neural Information Processing Systems}.\medskip

\noindent Petersen, P. and F. Voigtlaender (2017). Optimal approximation of piecewise smooth functions using deep ReLU neural networks. \emph{arXiv:1709.05289}.\medskip

\noindent Petrushev, P. P. (1999). Approximation by ridge functions and neural networks. \emph{SIAM J. Math Anal.}, 30, 155-189.\medskip

\noindent Pinkus, A. (1999). Approximation theory of the MLP model is neural networks. \emph{Acta Numerica}, 143-195.\medskip

\noindent Poggio, T., Mhaskar, H., Rosasco, L., Miranda, B., and Liao, Q. (2017). Why and when can deep-but not shallow-networks avoid the curse of dimensionality: A review.  
\emph{International Journal of Automation and Computing}, 14, 503-519.\medskip

\noindent Polson, N.  and Sokolov, V. (2017). Deep Learning: a Bayesian perspective.  
\emph{Bayesian Analysis}, 12, 1275-1304.\medskip

\noindent Rockova, V. and van der Pas, S. (2017). Posterior Concentration for Bayesian Regression Trees and their Ensembles. \emph{arXiv:1708.08734}.\medskip



\noindent Schmidt-Hieber, J. (2017). Nonparametric regression using deep neural networks with ReLU activation function. \emph{arXiv:1708.06633}.\medskip

\noindent Shen, X. and Wasserman, L. (2001). Rates of convergence of posterior distributions. \emph{Annals of Statistics}, 29, 687-714.\medskip

\noindent Srivastava, N., Hinton, G., Krizhevsky, A.,Sutskever, I. and Salakhutdinov, R. (2015). Dropout: a simple way to prevent neural networks from overfitting. \emph{Journal of Machine Learning Research}, 15, 1929-1958.\medskip

\noindent Telgarsky, M. (2016). Benefits of depth in neural networks. \emph{JMLR: Workshop and Conference Proceedings}, 49,1-23.\medskip

\noindent Telgarsky, M. (2017). Neural Networks and Rational functions. \emph{arXiv:1706.03301}.\medskip

\noindent Vitushkin, A. G. (1964). Proof of the existence of analytic functions of several complex variables which are not representable by linear superpositions of continuously differentiable 
functions of fewer variables.  \emph{Soviet Mathematics}, 5, 793-796.\medskip

\noindent Walker, S.,  Lijoi, A. and  Prunster, I. (2007). On rates of Convergence of Posterior Distributions in Infinite Dimensional Models. \emph{Annals of Statistics}, 35, 738-746.\medskip

\noindent Wager, S., Wang, S.  and Liang, P. (2014). Dropout training as adaptive regularization. \emph{Advances in Neural Information Processing Systems}.\medskip

\noindent Wong, W. H. and X. Shen (1995). Probability inequalities for Likelihood ratios and convergence rates of sieve mles. \emph{Annals of Statistics}, 23, 339-362.\medskip

\noindent Yarotsky, D. (2017). Error bounds for approximations with deep ReLU networks.  \emph{Neural Networks}, 94, 103-114.\medskip

\end{document}